\documentclass[letterpaper, 10 pt, conference]{ieeeconf}
\IEEEoverridecommandlockouts
\usepackage{textcomp}
\usepackage{xcolor}

\newcommand{\new}[1]{{#1}}
\newcommand{\newtwo}[1]{{#1}}

\usepackage[style=ieee]{biblatex}
\usepackage{amsmath,amssymb,amsfonts}
\usepackage{mathtools}
\newtheorem{theorem}{Theorem}

\newtheorem{lemma}{Lemma}
\newtheorem{definition}{Definition}
\newtheorem{example}{Example}

\usepackage{algorithmic}
\usepackage{algorithm}
\usepackage{graphicx}
\usepackage{hyperref}
\usepackage{tikz}
\usepackage{listofitems} % for \readlist to create arrays
\usepackage[outline]{contour} % glow around text
\contourlength{1.4pt}
\usetikzlibrary{angles, quotes, arrows.meta, matrix}
\colorlet{myred}{red!80!black}
\colorlet{myblue}{blue!80!black}
\colorlet{mygreen}{green!60!black}
\colorlet{myorange}{orange!70!red!60!black}
\colorlet{mydarkred}{red!30!black}
\colorlet{mydarkblue}{blue!40!black}
\colorlet{mydarkgreen}{green!30!black}

\tikzset{
  >=latex, % for default LaTeX arrow head
  node/.style={thick,circle,draw=myblue,minimum size=22,inner sep=0.5,outer sep=0.6},
  node in/.style={node,green!20!black,draw=mygreen!30!black,fill=mygreen!25},
  node hidden/.style={node,blue!20!black,draw=myblue!30!black,fill=myblue!20},
  node convol/.style={node,orange!20!black,draw=myorange!30!black,fill=myorange!20},
  node out/.style={node,red!20!black,draw=myred!30!black,fill=myred!20},
  connect/.style={thick,mydarkblue}, %,line cap=round
  connect arrow/.style={-{Latex[length=4,width=3.5]},thick,mydarkblue,shorten <=0.5,shorten >=1},
  node 1/.style={node in}, % node styles, numbered for easy mapping with \nstyle
  node 2/.style={node hidden},
  node 3/.style={node out}
}
\def\nstyle{int(\lay<\Nnodlen?min(2,\lay):3)} % map layer number onto 1, 2, or 3

\begin{document}

\title{\LARGE \bf Exploiting Symmetry in Dynamics for Model-Based Reinforcement Learning with Asymmetric Rewards}
% \author{Yasin Sonmez, Neelay Junnarkar, Murat Arcak% <-this % stops a space
% \thanks{Supported by the National Science Foundation award CNS-2111688 and  Air Force Office of Scientific Research grant FA9550-21-1-0288.
% %\neelay{TODO}
% }% <-this % stops a space
% \thanks{The authors are with the Electrical Engineering and Computer Sciences Dept., University of California, Berkeley. %Berkeley, CA 94720 USA 
% Emails:
%         \{{\tt\small yasin\_sonmez}, {\tt\small neelay.junnarkar}, {\tt\small arcak
% \}@berkeley.edu}.}
%         %
% }
\author{Yasin Sonmez, Neelay Junnarkar, and Murat Arcak
\thanks{This work was in supported in part by the National Science Foundation award CNS-2111688 and Air Force Office of Scientific Research grant FA9550-21-1-0288.}
\thanks{Yasin Sonmez, Neelay Junnarkar, and Murat Arcak are with the department
of Electrical Engineering and Computer Sciences at the University of California, Berkeley. (emails: \{yasin\_sonmez, neelay.junnarkar, arcak\} @berkeley.edu)}}

\pagestyle{empty}
\maketitle
\thispagestyle{empty}

%%%%%%%%%%%%%%%%%%%%%%%%%%%%%%%%%%%%%%%%%%%%%%%%%%%%%%%%%%%%%%%%%%%%%%%%%%%%%%%%
\begin{abstract}
Recent work in reinforcement learning has leveraged symmetries in the model to improve sample efficiency in training a policy. A commonly used simplifying assumption is that the dynamics and reward both exhibit the same symmetry; however, in many real-world environments, the dynamical model exhibits symmetry independent of the reward model.
%: the reward may not satisfy the same symmetries as the dynamics. 
In this paper, we %investigate 
%scenarios where 
assume
only the dynamics %are assumed to 
exhibit symmetry, extending the scope of problems in reinforcement learning and learning in control theory to which symmetry techniques can be applied. We use Cartan's moving frame method to introduce a technique for learning dynamics that, by construction, exhibit specified symmetries. %We demonstrate through 
 Numerical experiments 
 demonstrate
 that the proposed method learns a more accurate dynamical model.
\end{abstract}

% \begin{IEEEkeywords}
% Symmetry, neural networks, reinforcement learning.
% \end{IEEEkeywords}

%%%%%%%%%%%%%%%%%%%%%%%%%%%%%%%%%%%%%%%%%%%%%%%%%%%%%%%%%%%%%%%%%%%%%%%%%%%%%%%%

\section{Introduction}

A
%significant area of 
major research
area 
in reinforcement learning (RL) is %in 
improving sample efficiency, the amount 
%a learning method must interact with
of interaction with the environment 
needed
to learn a good policy. While model-free methods such as \cite{lillicrap2019continuous, haarnoja_soft_2018,schulman2017proximal} have shown the ability to achieve high rewards in many complex environments, they %typically 
require many environment interactions. Model-based reinforcement learning (MBRL) methods, on the other hand, have shown promise in learning policies quickly relative to the number of environment interactions. 
This 
 improvement in sample efficiency confirms that models 
%of systems 
contain a great deal of information useful for improving policies \cite{moerland_mbrlsurvey}.

%One class of RL methods that incorporate model-based information that has seen recent interest is the use of symmetries in models.
One idea gaining recent attention in RL methods is to integrate symmetries into model-based information.
% A symmetry refers to a transformation that relates a behavior (e.g. tuple $(s, a, s^\prime)$) to another behavior. 
Using knowledge of symmetries,
%a model satisfies, observed behavior can be 
one can extrapolate
observed behavior
to other scenarios.
% If one observes some behavior in the environment, and knows the symmetries which the model satisfies, then one not only knows the observed data, but also can extrapolate behavior to other scenarios using the transformations.
Recent work has leveraged symmetries via equivariant networks \cite{vanderpol2021mdp, wang_mathrmso2-equivariant_2022}, which enforce certain transformation relationships on the inputs and outputs, and via data augmentation \cite{weissenbacher2022koopman}, where artificial data is created based on observed data and known transformations for training. However, a limitation of equivariant networks is that they must be designed for the particular group of symmetries to which they 
are
%should be 
equivariant.
%Whereas 
Unlike the above methods, which assume that the symmetries are known apriori, 
\cite{weissenbacher2022koopman} learns symmetries from data and 
%leverages the learned symmetries by augmenting 
augments the training data using %the learned 
these symmetries.

Methods that leverage symmetry 
have also been explored for
%are also explored in the context of 
control
design and verification.
Reference \cite{grizzleOptimalControlSystems1984} uses symmetry to deduce the structure of the optimal controllers for nonlinear systems.
Reference \cite{SymmetryPreservingObservers} uses symmetry to improve the convergence of observers.
Reference \cite{maidens_dynamic_programming} explores the use of symmetry for reducing the dimension of variables involved in dynamic programming, accelerating the computation of control policies.
Reference \cite{maidens_reachability} uses symmetry reduction to accelerate reachability computations.
References \cite{sibaiMultiagentSafetyVerification2020} and \cite{sibaiUsingSymmetryTransformations2019} use symmetry in accelerating safety verification. 
% {\color{orange} Add one of Hussein's papers?}

A common assumption 
%in the use of 
when using
symmetry in 
%both 
RL and optimal control is that both the dynamics and reward (or cost) function exhibit the same symmetries,
% The majority of these works on the use of symmetry in RL have explored the case where both the dynamics and reward exhibit the same symmetry.
%This combination of assumptions 
which
results in the optimal policy (or control law) 
%exhibiting the same symmetry 
likewise being symmetric
\cite{ravindran_model_2002, maidens_dynamic_programming}.
The symmetric policy structure is exploited 
%in the context of 
for Markov Decision Processes (MDPs) and RL in \cite{vanderpol2021mdp, wang_mathrmso2-equivariant_2022, ravindran_model_2002, zinkevich_symmetry_2001}, and for dynamic programming in \cite{maidens_dynamic_programming}.
In many scenarios, however, symmetries in the dynamical model are not necessarily applicable to the reward model.
% For example, in a dog robot, the legs may be constructed the same, and thus each leg has the same dynamical model, incorporated in different ways into the overall dynamical model. This has no relation to the reward model. 
For example, multiple tasks might be executed in the same environment, where the same dynamical symmetries exist, while the reward model changes from task to task.
Exploiting symmetries only in the dynamical model widens the scope of RL problems to which symmetry can be applied.
%, and could potentially lead to other benefits such as transfer learning (\cite{moerland_mbrlsurvey}).

%Inspired 
Motivated by commonly encountered translation and rotation symmetries in dynamics, we consider a method of enforcing continuous symmetries through transformation into a reduced coordinate space. 
% in the training data of the dynamical model.
For example, the dynamics of the cart-pole system, commonly used for control systems pedagogy, are invariant to the position of the cart. 
This 
%could be trivially enforced 
fact can be exploited
when learning the system dynamics by simply removing the position data from the training dataset.
% Applications include model-based reinforcement learning.
%While exploiting this particular symmetry is simple, it 
It is, however, less clear how to account for symmetries that involve more complex transformations.
% such as combinations of rotation and translation.
% {\color{orange} Here I would indicate that how we can exploit other symmetries is less clear. That would set the stage for the next paragraph introducing the contribution.}

In this paper, we present a method to learn dynamical models that are, by construction, invariant under specified continuous symmetries. 
This enables encoding apriori known symmetry structure when learning a dynamical model.
To address a large class of symmetries, we use Cartan's moving frame method \cite{cartan1937théorie}, which simplifies models by adapting moving reference frames to the structure of the space.
% by transform coordinate systems.

% {\color{orange} I think Cartan should be mentioned here as the solution to the problem raised in the previous paragraph.}
% Our method applies to continuous symmetries.

This paper is organized as follows: in Section~\ref{sec:symmetries}, we provide background information on symmetries, and then use Cartan's method of moving frames to relate a symmetrical dynamical model to a function operating on an input space of reduced dimension. 
In Section~\ref{sec:learning}, we 
use
%illustrate how 
the result of Section~\ref{sec:symmetries} 
%can be used 
to learn a dynamical model that, by construction, satisfies 
%certain 
specified 
symmetries.
In Section~\ref{sec:experiments} we apply symmetry reduction to learn dynamical models for two examples.
\section{Symmetries in Model Dynamics} \label{sec:symmetries}

Consider the deterministic discrete-time system
\begin{equation} \label{eq:discrete-dynamics}
    x_{k+1} = F(x_k, u_k)
\end{equation}
where $x_k \in \mathcal{X} = \mathbb{R}^{n}$
is the state, $u_k \in \mathcal{U} = \mathbb{R}^{n_u}$
is the control input, and $k \in \mathbb{N}$ denotes the time step.
\new{ \textit{Symmetry} %we loosely mean 
refers to a
structure of the dynamical system such that transformations of the input $(x_k, u_k)$ to $F$ result in corresponding, known, transformations to the output $x_{k+1}$.
We formalize this notion with \textit{Lie transformation groups}, an introduction to which can be found in \cite{bulloGeometricControlMechanical2004} and \cite{leeIntroductionSmoothManifolds2012}.
\begin{definition}[Lie Transformation Group]
    Let $G$ be a Lie group acting on a smooth manifold $\Sigma$.
    Denote the group action by $w(g, \sigma) \mapsto g \cdot \sigma$.
    The group $G$ is a Lie transformation group if $w: G \times \Sigma \to \Sigma$ is smooth.
\end{definition}

Important properties of the group action are:
\begin{itemize}
    \item $e \cdot \sigma = \sigma$ for all $\sigma$ where $e$ is the identity of group $G$.
    \item $g_1 \cdot (g_2 \cdot \sigma) = (g_1 * g_2) \cdot \sigma$ for all $g_1, g_2 \in G, \sigma \in \Sigma$, where $g_1 * g_2$ denotes the group operation, which we will denote by $g_1 g_2$ henceforth.
\end{itemize}

We represent symmetries of \eqref{eq:discrete-dynamics} with an $r$-dimensional Lie group $G$ which is a Lie transformation group of both $\mathcal{X}$ and $\mathcal{U}$.
%($G$ acts on both $\mathcal{X}$ and $\mathcal{U}$, and these actions are smooth)
Assume $r \leq n$.
For notation, we consider the maps $\phi_g: \mathcal{X} \to \mathcal{X}$ defined by $x \mapsto g \cdot x$ and $\psi_g: \mathcal{U} \to \mathcal{U}$ defined by $u \mapsto g \cdot u$. 
Since the group actions are smooth, both of these maps are smooth for all $g \in G$.
Further, they 
%are invertible (with smooth inverses at that)
%with 
have smooth
inverses $\phi^{-1}_g = \phi_{g^{-1}}$ and %similar for $\psi_g$.
$\psi^{-1}_g = \psi_{g^{-1}}$.

}
We next %introduce a definition of 
define
dynamics invariant to symmetry: %inspired by \cite{maidens_reachability}.
\begin{definition}[Invariant Dynamics]
    The system \eqref{eq:discrete-dynamics} is $G$-invariant if $F(\phi_g(x), \psi_g(u)) = \phi_g(F(x,u))$ for all $g \in G$, $x \in \mathcal{X}$, $u \in \mathcal{U}$.
\end{definition}

\new{

The premise of our method is to transform 
%an input 
$(x, u)$ to its canonical form, corresponding to an element of $\{(\phi_g(x), \psi_g(u))\ |\ g \in G\}$.
%We show that when 
When $F$ in \eqref{eq:discrete-dynamics} is $G$-invariant, 
%it can be 
we parameterize it by a function that operates on these canonical forms, which reside in a lower-dimensional space.
}
\new{
\subsection{Cartan's Moving Frame}
We use Cartan's moving frames \cite{cartan1937théorie} to construct a map from elements of the state space to their canonical forms.
What follows is a brief recap of Cartan's method based on \cite{SymmetryPreservingObservers, maidens_reachability, olverClassicalInvariantTheory1999}.
Note that, in general, this method constructs a map which exists only locally, due to the use of the implicit function theorem.
On many practical examples, such as those in this work and \cite{maidens_reachability}, the map exists over all $\mathcal{X}$ except a lower-dimensional submanifold.
Thus, in our presentation, we will focus on transformations within a single smooth chart.
Refer to \cite{olverClassicalInvariantTheory1999} for a thorough treatment of the subject.

Consider the Lie group $G$ and its action on $\mathcal{X}$ as before.
Assume $\mathcal{X}$ can be split into $\mathcal{X}^a$ and $\mathcal{X}^b$ with $r$ and $n-r$ dimensions respectively such that, for some $x_0 \in \mathcal{X}$, the map from $g \in G$ to the projection of $g \cdot x_0$ onto $\mathcal{X}^a$ is invertible.
Pick $c \in \mathcal{X}^a$ in the range of this map.
% For all $g \in G$, let $\phi_g^a: \mathcal{X} \to \mathcal{X}^a$ and $\phi_g^b: \mathcal{X} \to \mathcal{X}^b$ denote the projections of $\phi_g$ onto $\mathcal{X}^a$ and $\mathcal{X}^b$ respectively.
}
\newtwo{
For all $g \in G$, let $\phi_g^a: \mathcal{X} \to \mathcal{X}^a$ and $\phi_g^b: \mathcal{X} \to \mathcal{X}^b$ denote the compositions of $\phi_g$ and the projections onto $\mathcal{X}^a$ and $\mathcal{X}^b$ respectively.
}
\new{
Assume that, for all $x \in \mathcal{X}$, there exists a unique $g \in G$ such that $\phi_g^a(x) = c$.
%Define 
The set $\mathcal{C} = \{x | \phi_e^a(x) = c\}$
 is an $n-r$ dimensional submanifold of $\mathcal{X}$ called a \textit{cross-section}.
We 
can 
then define 
%a
the 
unique
map $\gamma: \mathcal{X} \to G$, such that $\phi_{\gamma(x)}(x) \in \mathcal{C}$.
This $\gamma$ is
called a \textit{moving frame}.
Using $\gamma$, we define 
the map 
$\rho: \mathcal{X} \to \mathcal{X}^b$ from elements in $\mathcal{X}$ to their canonical forms:
}
\begin{equation} \label{eq:rho}
    \rho(x) \triangleq \phi_{\gamma(x)}^b(x).
\end{equation}

This $\rho$ is invariant to the action of $G$ on the state, as shown in the following lemma adapted from \cite{SymmetryPreservingObservers}.
\begin{lemma} \label{lemma:rho}
    For all $g \in G$ and $x \in \mathcal{X}$, $\rho(\phi_g(x)) = \rho(x)$.
\end{lemma}
\begin{proof}
    Let $g\in G$.
    Then,
    $\rho(\phi_g(x)) = \phi_{\gamma(\phi_g(x))}^b(\phi_g(x)) = \phi_{\gamma(\phi_g(x)) g}^b(x)$.
    % Then $\rho(\phi_g(x)) = \phi_{\gamma(\phi_g(x))}^b(\phi_g(x)) = \phi_{\gamma(\phi_g(x)) g}^b(x)$.
    \new{From the definition of $\gamma$, we have that both $\phi_{\gamma(x)}^a(x)$ and $\phi_{\gamma(\phi_g(x))}^a(\phi_g(x))$ equal $c$.
    Note that $\phi_{\gamma(\phi_g(x))}^a(\phi_g(x)) = \phi_{\gamma(\phi_g(x)) g}^a(x)$.
    Since $\gamma(x)$ is the unique group element $h$ such that $\phi_h^a(x) = c$, we have that $\gamma(\phi_g(x)) g = \gamma(x)$.
    }
    Plugging this into the expression for $\rho(\phi_g(x))$ gives $\rho(\phi_g(x)) = \phi_{\gamma(x)}^b(x) = \rho(x)$.
\end{proof}

\new{
\example{ \label{example:car}
    In this example, we demonstrate solving for the moving frame for a system with both translation and rotation symmetries.
    Consider a car with state $x = (y, z, v_y, v_z, h_y, h_z)$ where $y$ and $z$ correspond to positions in a plane, $v_y$ and $v_z$ correspond to velocities along the $y$ and $z$ directions, and $h_y$ and $h_z$ correspond to the cosine and sine of the heading angle (and thus $h_y^2 + h_z^2 = 1$).
    Since we assume the dynamics of the car are both translation- and rotation-invariant, we consider the special Euclidean group $G = \mathrm{SE}(2)$, and parameterize it by a translation $(\tilde{y}, \tilde{z})$ and a rotation $\tilde{\theta}$.
    Then, with rotation matrix $R_{\tilde{\theta}} = \begin{bsmallmatrix} \cos \tilde{\theta} & -\sin \tilde{\theta} \\
            \sin \tilde{\theta} & \cos \tilde{\theta} \end{bsmallmatrix}$,
    % \begin{equation}
    %     R_{\theta^\prime} = \begin{bmatrix}
    %         \cos \theta^\prime & -\sin \theta^\prime \\
    %         \sin \theta^\prime & \cos \theta^\prime
    %     \end{bmatrix},
    % \end{equation}
    we define the action of $G$ on the state space by
    \begin{equation}
        \phi_{(\tilde{y}, \tilde{z}, \tilde{\theta})}(x) = \begin{bmatrix}
            R_{\tilde{\theta}} & 0 & 0 \\
            0 & R_{\tilde{\theta}} & 0 \\
            0 & 0 & R_{\tilde{\theta}}
        \end{bmatrix}
        x + \begin{bmatrix}
            \begin{bmatrix}
                \tilde{y}  \\ \tilde{z} 
            \end{bmatrix} \\ 0 \\ 0
        \end{bmatrix}.
    \end{equation}

    We now derive $\gamma$ and $\rho$.
    Let $\mathcal{X}^a$ be the projection of $\mathcal{X}$ onto the $(y, z, h_y, h_z)$ components.
    Pick $c = (0, 0, 1, 0)$, representing a position at the origin with a heading angle of $0$.
    There exists a unique $(\tilde{y}, \tilde{z}, \tilde{\theta})$ to transform each $x \in \mathcal{X}$ into the cross-section defined by $\mathcal{C}$.
    We solve the following system of equations to find the moving frame:
    \begin{equation*}
        \begin{bmatrix}
            0 \\ 0 \\ 1 \\ 0
        \end{bmatrix} = \phi_{\gamma(x)}^a(x) =
        \begin{bmatrix}
            R_{\tilde{\theta}} & 0 \\ 0 & R_{\tilde{\theta}}
        \end{bmatrix} \begin{bmatrix}
            y \\ z \\ h_y \\ h_z
        \end{bmatrix} +
        \begin{bmatrix}
            \tilde{y} \\ \tilde{z}  \\ 0 \\ 0
        \end{bmatrix}.
    \end{equation*}
    This gives the solution
    \begin{equation}
        \gamma(x) = \begin{bmatrix} \tilde{y} \\ \tilde{z} \\ \tilde{\theta} \end{bmatrix} = \begin{bmatrix}
            -y h_y - z h_z \\ 
            y h_z - z h_y \\
            \mathrm{arctan2}(-h_z, h_y)
        \end{bmatrix}.
    \end{equation}
    We can then compute the group inverse of $\gamma(x)$ as:
    \begin{equation}
        \gamma(x)^{-1} = \begin{bmatrix}
            y \\ z \\ \mathrm{arctan2}(h_z, h_y)
        \end{bmatrix}.
    \end{equation}
    Substituting $\gamma(x)$ into the definition of $\rho$ from \eqref{eq:rho}, we get:
    \begin{equation}
        \rho(x) = \begin{bmatrix}
            h_y v_y + h_z v_z \\
            -h_z v_y + h_y v_z
        \end{bmatrix}.
    \end{equation}
}

}

\new{
\subsection{Dimension Reduction}
Using the map $\rho$ from elements in $\mathcal{X}$ to their canonical forms, we show that the dynamics in \eqref{eq:discrete-dynamics} are equivalent to a function operating on a lower-dimensional input space.
}

% The map $\rho$ outputs the element of $\mathcal{C}$, eliminating the part equal to $c$, that is related by a symmetry transformation to the state $x$.
% This effectively transforms the state space from dimension $n$ to dimension $n-r$.
% The following theorem shows the dynamics can be evaluated using the lower dimensional output of $\rho$.

\begin{theorem} \label{thm:state-reduction}
    %The system \eqref{eq:discrete-dynamics}, 
    $F: \mathcal{X} \times \mathcal{U} \to \mathcal{X}$
    in system \eqref{eq:discrete-dynamics}
     is $G$-invariant if and only if there exists 
    %a map
    $\bar{F}: \mathcal{X}^b \times \mathcal{U} \to \mathcal{X}$ such that 
    \begin{equation} \label{eq:thm-fbar-statement}
    \phi_{\gamma(x)}(F(x,u)) = \bar{F}(\rho(x), \psi_{\gamma(x)}(u))
    \end{equation}
    for all $x \in \mathcal{X}$ and $u \in \mathcal{U}$.
\end{theorem}
\begin{proof}
    Assume that the dynamics $F$ are $G$-invariant.
    Note that $\rho$ restricted to $\mathcal{C}$ is invertible.
    Define $\bar{F}$ by $\bar{F}(\bar{x}, u) = F(\rho|_{\mathcal{C}}^{-1}(\bar{x}), u)$.
    Letting $x = \rho|_{\mathcal{C}}^{-1}(\bar{x})$, this can equivalently be written as $\bar{F}(\rho(x), u) = F(x, u)$ when $x \in \mathcal{C}$.
    Since $F$ is $G$-invariant, $\phi_{\gamma(x)}(F(x,u)) = F(\phi_{\gamma(x)}(x), \psi_{\gamma(x)}(u))$.
    Note that $\phi_{\gamma(x)}(x) \in \mathcal{C}$.
    This can be rewritten in terms of $\bar{F}$ as $\bar{F}(\rho(\phi_{\gamma(x)}(x)), \psi_{\gamma(x)}(u))$.
    By Lemma~\ref{lemma:rho}, this equals $\bar{F}(\rho(x), \psi_{\gamma(x)}(u))$, as desired.

%    Now we 
To prove the reverse direction,
    assume 
     %such  $\bar{F}$.
    $\bar{F}: \mathcal{X}^b \times \mathcal{U} \to \mathcal{X}$ exists such that
    \eqref{eq:thm-fbar-statement} holds
    % $\phi_{\gamma(x)}(F(x,u)) = \bar{F}(\rho(x), \psi_{\gamma(x)}(u))$
    for all $x \in \mathcal{X}$, $u \in \mathcal{U}$.
    Let $g \in G$. 
    Then,
    \begin{align*}
        % \phi_g(F(x,u)) & = \phi_g(\phi_{\gamma(x)}^{-1}(\bar{F}(\rho(x), \psi_{\gamma(x)}(u)))) \\
        % & = 
        F(& \phi_g(x), \psi_g(u))  \\
        & = \phi_{\gamma(\phi_g(x))}^{-1}(\bar{F}(\rho(\phi_g(x)), \psi_{\gamma(\phi_g(x))}(\psi_g(u)))) \\
        & = \phi_{\gamma(\phi_g(x))}^{-1}(\bar{F}(\rho(x), \psi_{\gamma(\phi_g(x))g}(u))). \\
    \end{align*}
    From the proof of Lemma~\ref{lemma:rho}, we know $\gamma(\phi_g(x))g = \gamma(x)$.
    Therefore, the above can be simplified to 
    \begin{align*}
        F(\phi_g(x), \psi_g(u)) & = \phi_{\gamma(x)g^{-1}}^{-1}(\bar{F}(\rho(x), \psi_{\gamma(x)}(u))) \\
        & = \phi_{g} (\phi_{\gamma(x)^{-1}} ( \phi_{\gamma(x)} ( F(x,u) ))) \\
        & 
        = \phi_{g} ( F(x,u) ).
    \end{align*}
    Thus, $F$ is $G$-invariant.
\end{proof}

\section{Learning Model Dynamics} \label{sec:learning}

A critical part of the theory in Section~\ref{sec:symmetries} is that the \new{ only knowledge of system dynamics required is that of the state space, the control input space, and the system symmetries.}
% we do not require knowledge of the system dynamics.
We leverage this to
%learn a dynamical model that exhibits specified symmetries.
%, instead of knowing the dynamical model and deriving symmetries from there.
% \neelay{Say something about why we would know symmetries and not model dynamics. }
learn a dynamical model that is invariant to a given transformation group 
%through application of 
using Theorem~\ref{thm:state-reduction}.
Unlike a standard formulation, 
%for learning a dynamical model, 
which would learn the 
%dynamical 
model $F: \mathcal{X} \times \mathcal{U} \to \mathcal{X}$, we learn a function $\bar{F}: \mathcal{X}^b \times \mathcal{U} \to \mathcal{X}$ and
% from the statement of Theorem~\ref{thm:state-reduction}.
% Instead of directly learning the $F$ in \eqref{eq:discrete-dynamics}, we aim to learn the $\bar{F}$ from the statement of Theorem~\ref{thm:state-reduction}.
%With this $\bar{F}$, we 
construct 
%the model dynamics 
$F$ from:
\begin{equation} \label{eq:recover-F}
    F(x, u) \triangleq \phi_{\gamma(x)}^{-1}(\bar{F}(\rho(x), \psi_{\gamma(x)}(u))).
\end{equation}
\new{Note that, in this section, we use 
%the triplet 
$(x, u, x^\prime)$ to denote a state, an action applied to the state, and the next state.}

The relationship between $F$ and $\bar{F}$ is depicted in Figure~\ref{fig:F-Fbar-relationship}.
From Theorem~\ref{thm:state-reduction},  $F$ is, by construction, $G$-invariant.
A side-benefit is that training is done with a lower dimensional input space ($\mathcal{X}^b \times \mathcal{U}$ instead of $\mathcal{X} \times \mathcal{U}$).
% \new{
% Given a dataset consisting of tuples $(x_k, u_k, x_{k+1})$, we train $\bar{F}$ to map $(\rho(x_k), \psi_{\gamma(x_k)}(u_k))$ to $\phi_{\gamma(x_k)}(x_{k+1})$ by using $\bar{F}$ to compute $F$ and then using the original objective function, leveraging differentiation through the smooth maps $\phi_g, \rho, \psi_g$.
% }
To learn this function $\bar{F}$, given a dataset $\mathcal{D}$ consisting of tuples $(x, u, x^\prime)$, we train $\bar{F}$ to map $(\rho(x), \psi_{\gamma(x)}(u))$ to $\phi_{\gamma(x)}(x^\prime)$.
% Evaluating $F(x, u)$ through \eqref{eq:recover-F} can be interpreted as first bringing $(x, u)$ into a reduced coordinate space, something somethign something

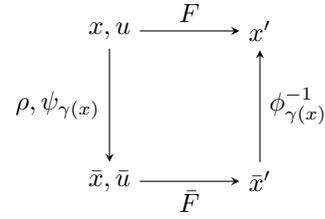
\begin{figure}[t]
    \centering
    \begin{tikzpicture}[>=stealth, node distance=2cm, auto]
        % Nodes
        \node (x_u) {\(x, u\)};
        \node (xprime) [right of=x_u] {\(x^\prime\)};
        \node (xbar_ubar) [below of=x_u] {\(\bar{x}, \bar{u}\)};
        \node (xbarprime) [below of=xprime] {\(\bar{x}^\prime\)};
    
        % Arrows
        \draw[->] (x_u) to node {\(F\)} (xprime);
        \draw[->] (x_u) to node [swap] {\(\rho, \psi_{\gamma(x)}\)} (xbar_ubar);
        \draw[->] (xbarprime) to node [swap] {\(\phi_{\gamma(x)}^{-1}\)} (xprime);
        \draw[->] (xbar_ubar) to node [swap] {\(\bar{F}\)} (xbarprime);
    \end{tikzpicture}
    \caption{Relationship between $F$ and $\bar{F}$.}
    \label{fig:F-Fbar-relationship}
\end{figure}

Note that, in practice, it is typical to learn a function which maps $(x, u)$ to $x^\prime - x$ instead of $x^\prime$ directly.
% \neelay{(cite something why)}.
To do this, we learn $\Delta \bar{F}$ to map $(\rho(x), \psi_{\gamma(x)}(u))$ to $\phi_{\gamma(x)}(x^\prime) - \phi_{\gamma(x)}(x)$.
Then, we construct $F$ by 
\begin{equation*}
    F(x,u) = \phi_{\gamma(x)}^{-1}(\Delta\bar{F}(\rho(x), \psi_{\gamma(x)}(u)) + \phi_{\gamma(x)}(x)).
\end{equation*}
This construction for $F$ is $G$-invariant since it is equivalent to using $\bar{F}(\bar{x}, \bar{u}) \triangleq \Delta\bar{F}(\bar{x}, \bar{u}) + \rho|_\mathcal{C}^{-1}(\bar{x})$, which satisfies the condition in Theorem~\ref{thm:state-reduction}.
% The proof of this construction making $F$ be $G$-invariant is similar to the proof in Theorem~\ref{thm:state-reduction}.
Finally, $\Delta F(x, u)$ can be defined as $\Delta F(x , u) \triangleq F(x, u) - x$ to provide the same $x^\prime - x$ interface expected by many libraries.
This construction for $\Delta F$ in terms of $\Delta \bar{F}$ expands as follows.
\begin{equation} \label{eq:delta-F}
    \Delta F(x, u) \triangleq \phi_{\gamma(x)}^{-1}(\Delta \bar{F}(\rho(x), \psi_{\gamma(x)}(u)) + \phi_{\gamma(x)}(x)) - x
\end{equation}

One special case where this simplifies is when $\phi_g$ is a homomorphism with respect to $+$ for all $g \in G$. 
Then, 
\eqref{eq:delta-F} simplifies to
\begin{equation*}
    \Delta F(x, u) = \phi_{\gamma(x)}^{-1} (\Delta \bar{F}(\rho(x), \psi_{\gamma(x)}(u))).
\end{equation*}
An example where this occurs is when $\phi_g$ is linear.

%%%%%%%%%%%%%%%%%%%%%%%%%%%%%%%%%%%%%%%%%%%%%%%%%%%%%%%%%%%%%%%%%%%%%%%%%%%%%%%%
\section{Experiments} \label{sec:experiments}

In this section, we demonstrate the performance of this method on two numerical examples involving learning the dynamics of a system. The two environments that we used were ``Parking" from ``Highway-env" \cite{highway-env} and ``Reacher" from OpenAI Gym \cite{brockman2016openai}.
These exhibit rotational and/or translational symmetry, as depicted in Figure \ref{fig:environments}. In ``Parking" there are two controlled vehicles that we need to park to corresponding parking spots, and in ``Reacher" we control a robot arm with 2 joints to reach to a target position.  To be able to compare our method fairly to a baseline we generated a static dataset of transitions ${(x, u, x^\prime)}$, and we learn a dynamical model from this dataset. Consequently, the comparison between the two methodologies was centered on their respective capabilities in dynamics learning rather than policy generation. This approach ensures a fair comparison, as it relies on the same offline dataset for both methods, avoiding discrepancies that may arise from online training and dataset variations.

The dataset is constructed by training an agent using Soft Actor-Critic \cite{haarnoja_soft_2018}, a model-free reinforcement learning method. 
We use D3rlpy \cite{d3rlpy} to train dynamical models from the pre-collected dataset.
We implement the symmetry method discussed above using \eqref{eq:recover-F}, where $\bar{F}$ is the neural network whose parameters are trained.
The observation error (on the test dataset) of this dynamical model is compared against the observation error of learning $F$ directly.
We further compare these methods (with and without symmetry) across a variety of parameter sizes.
The rollouts from the resulting policy, the datasets used in this work and the code to reproduce the results 
%in this paper 
%can be found 
%on 
%our 
are available in a
GitHub repository\footnote{\url{https://github.com/YasinSonmez/symmetry-cs285}}.

\subsection{Two Cars Parking Scenario} \label{subsec:parking}
We use the parking environment from \cite{highway-env} with two cars with separate goal positions. The aim %of this environment 
is to park the two cars at their goal %parking 
spots by controlling them separately and not 
allowing them to crash
%crashing the cars 
into obstacles or to each other. 
%Note that the 
The dynamics of each car are translation- and rotation-invariant, but the reward function (relating to the location and orientation of the parking spot) is not.
Therefore, 
%it is not possible to employ 
methods that assume 
%that 
both the dynamics and the reward function satisfy the same symmetries
are not applicable.

\begin{figure}[tp]
    \centering
    \includegraphics[width = \linewidth, clip, viewport=50 100 1250 500, decodearray={0.15 1 0.15 1 0.15 1}]{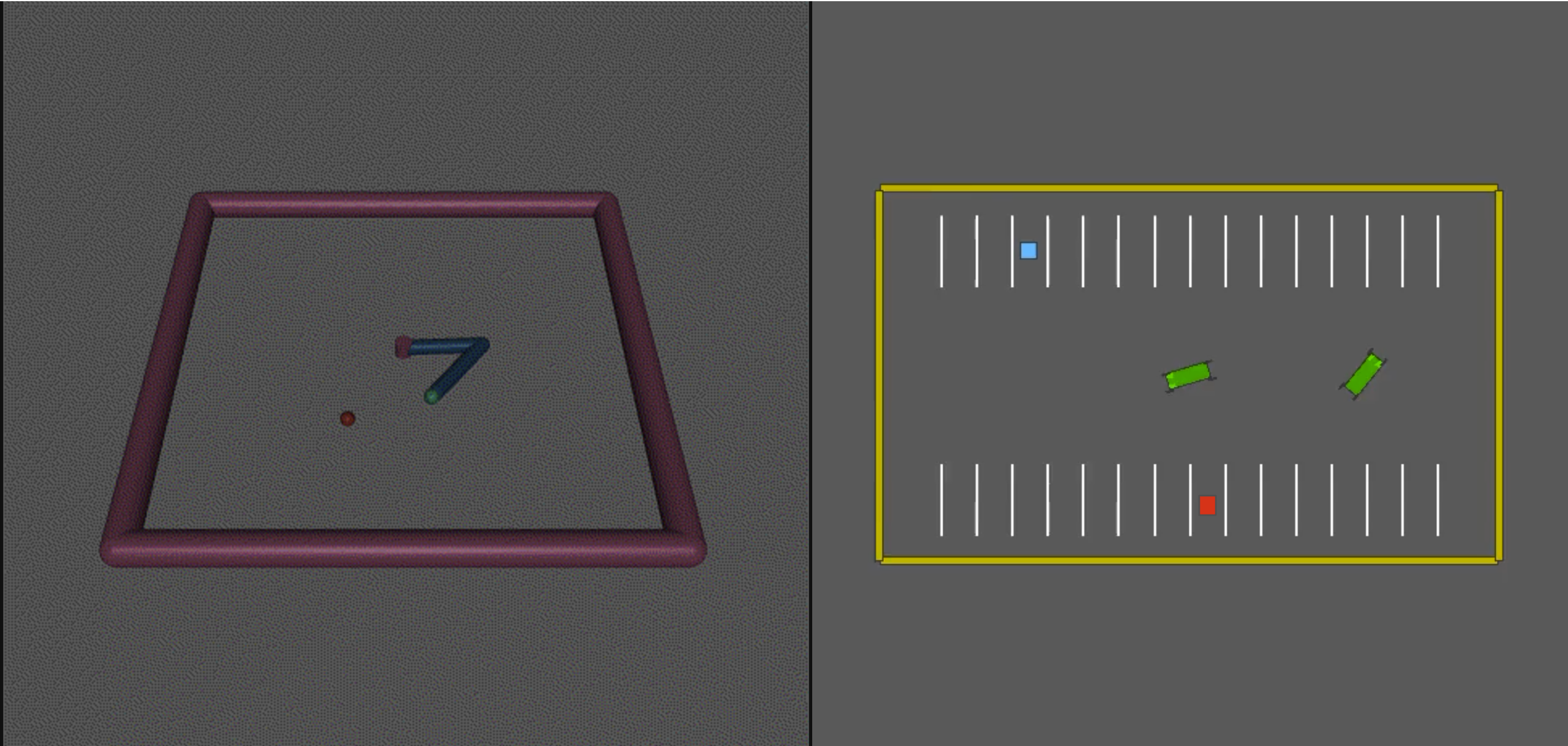}
    \caption{Experimental environments included: (1) ``Reacher'' on the left, featuring two rotating controlled joints that exhibit rotation symmetry with the objective of reaching a target point, and (2) ``Parking" on the right, involving two controlled vehicles maneuvering to park in designated spots without collision, demonstrating both rotational and translational symmetry in dynamics.}
    \label{fig:environments}
\end{figure}

% Note that the dynamics of each car are translation- and rotation-invariant, but the reward function (relating to the location and orientation of the parking spot) is not. 
% Therefore it is not possible to employ methods that assume that both the dynamics and the reward function are symmetric.
% % The car dynamics in this environment follow the kinematic bicycle model \cite{kinematicBicycle}, 
% Each car has state $ x = (y, z, v_y, v_z, h_y, h_z)$ corresponding to the $y$ position, $z$ position, $y$ component of velocity, $z$ component of velocity, cosine of heading angle, and sine of heading angle.

The state of this parking environment is 24-dimensional, with the first 6 states corresponding to the first car, the second 6 states corresponding to the second car, the third 6 states corresponding to the first car's goal, and the final 6 states corresponding to the second car's goal. The state %corresponds to 
incorporates
%4 independent systems: 
the two car models and two goal models.
The goals follow the 
%simple %dynamical 
model  $g_{k+1} = g_k$.
% the two cars follow the kinematic bicycle model, and the two goal systems are static, following the simple dynamical model of $g_{k+1} = g_k$.
\new{We include the goal states in the state representation because it aligns with the conventional RL framework where all states are utilized directly in the training process without explicit distinction.}
We apply the method presented in Sections \ref{sec:symmetries} and \ref{sec:learning} to construct transformation groups for each of the four systems.
The joint system then satisfies the transformation group formed as the product of the individual transformation groups.

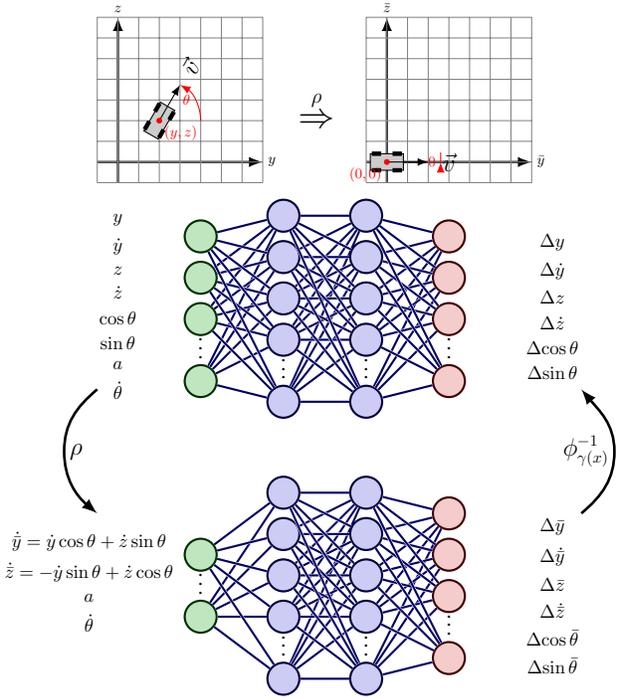
\begin{figure}[tp]
    \centering
    \begin{tikzpicture}[scale=0.55, transform shape]
    % Define car style with orientation
    \tikzset{
        car/.pic={
            % Car body
            \filldraw [fill=black!20!white, draw=black] (0.2, 0.4) rectangle (-0.2,-0.4);
            % Wheels
            \filldraw [fill=black] (-0.16,-0.15) rectangle (-0.24,-0.35);
            \filldraw [fill=black] (0.16,-0.15) rectangle (0.24,-0.35);
            \filldraw [fill=black] (-0.16,0.15) rectangle (-0.24,0.35);
            \filldraw [fill=black] (0.16,0.15) rectangle (0.24,0.35);
            \draw[->] (0,0) -- ++(90:1) node[pos=1.5,sloped=true] {\LARGE $\vec{v}$};
        }
    }
    % Original figure
    {
        % Drawing axes
        \draw[thick,->] (-0.5,0) -- (3.5,0) node[right] {$y$};
        \draw[thick,->] (0,-0.5) -- (0,3.5) node[above] {$z$};
        
        % Adding grid
        \draw[step=0.5, gray, very thin] (-0.5,-0.5) grid (3.5,3.5);
        
        % Points coordinates
        \def\y{1}
        \def\z{1}
        
        % Angles with the y-axis
        \def\mytheta{60} % Angle for the first arrow

        % Place a car at a specific coordinate and rotate it
        \pic[rotate=-90+\mytheta] at (\y,\z) {car}; % Rotate the car by 45 degrees
        
        % Marking the points
        \fill[red] (\y,\z) circle (2pt) node[below right] {$(y,z)$};
        
        % Drawing arcs for angles
        \draw[->, red] (\y,\z) +(right:1cm) arc[start angle=0, end angle=\mytheta, radius=1cm] node[midway, left] {$\theta$};
    }

    % Adjust node placement to correctly position \rho above the arrow
    \node (arrow) at (4.8,1) {\Huge $\Rightarrow$}; % Place the arrow
    \node at (4.8,1.5) {\Large $\rho$}; % Adjust the y-coordinate as needed to place \rho above the arrow
    % Duplicate figure
    \begin{scope}[xshift=6.5cm] % Shift horizontally to place next to the original
        % Replicate the content of the original figure
        % Drawing axes
        \draw[thick,->] (-0.5,0) -- (3.5,0) node[right] {$\bar{y}$};
        \draw[thick,->] (0,-0.5) -- (0,3.5) node[above] {$\bar{z}$};
        
        % Adding grid
        \draw[step=0.5, gray, very thin] (-0.5,-0.5) grid (3.5,3.5);
        
        % Points coordinates and the rest of the figure are the same as above
        % Since the content is identical, you can simply replicate the drawing commands here
        \def\y{0}
        \def\z{0}
        \def\theta{0}
        % Place a car at a specific coordinate and rotate it
        \pic[rotate=-90+\theta] at (\y,\z) {car}; % Rotate the car by 45 degrees
        
        \fill[red] (\y,\z) circle (2pt) node[below left] {$(0,0)$};
        \draw[->, red] (\y,\z) +(right:1.3cm) arc[start angle=0, end angle=\theta, radius=1cm] node[midway, left] {$0$};
        
    \end{scope}

    \begin{scope}[transform canvas={scale=0.7}]
        % Matrix with brackets indicating transformation
        \matrix (m) [matrix of math nodes, nodes={minimum width=0em, minimum height=0em}] at (0,-5)
        {
            y  \\
            \dot{y}  \\
            z  \\
            \dot{z}  \\
            \cos{\theta}  \\
            \sin{\theta}  \\
            a  \\
            \dot{\theta}  \\
        };
        % Matrix with brackets indicating transformation
        \matrix (m) [matrix of math nodes, nodes={minimum width=0em, minimum height=0em}] at (15,-5)
        {
            \new{\Delta{y}}  \\
            \new{\Delta{\dot{y}}} \\
            \new{\Delta{z}}  \\
            \new{\Delta{\dot{z}}}\\
            \new{\Delta{\cos{\theta}}} \\
            \new{\Delta{\sin{\theta}}}\\
        };
    \end{scope}
      % Left arrow with label
      \draw[->, line width=1pt] (-0.5,-5.5) to[out=220,in=90] node[right,pos=1] {\LARGE $\rho$} (-1.3,-7) to[out=270,in=140] (-0.5,-8.5);
      
      % Right arrow with label
      \draw[->, line width=1pt] (11.2,-8.5) to[out=40,in=270] node[left,pos=1] {\LARGE $\phi_{\gamma(x)}^{-1}$} (12,-7) to[out=90,in=310] (11.2,-5.5);
    
    \begin{scope}[yshift=-2.8cm] 
      \message{^^JNeural network, shifted}
      \readlist\Nnod{4,5,5,4} % array of number of nodes per layer
      \readlist\Nstr{n,m,m,m,k} % array of string number of nodes per layer
      \readlist\Cstr{\strut x,a^{(\prev)},a^{(\prev)},a^{(\prev)},y} % array of coefficient symbol per layer
      \def\yshift{0.5} % shift last node for dots
      
      \message{^^J  Layer}
      \foreachitem \N \in \Nnod{ % loop over layers
        \def\lay{\Ncnt} % alias of index of current layer
        \pgfmathsetmacro\prev{int(\Ncnt-1)} % nmber of previous layer
        \message{\lay,}
        \foreach \i [evaluate={\c=int(\i==\N); \y=\N/2-\i-\c*\yshift;
                     \index=(\i<\N?int(\i):"\Nstr[\lay]");
                     \x=2*\lay; \n=\nstyle;}] in {1,...,\N}{ % loop over nodes
          % NODES
          \node[node \n] (N\lay-\i) at (\x,\y) {};
          
          % CONNECTIONS
          \ifnum\lay>1 % connect to previous layer
            \foreach \j in {1,...,\Nnod[\prev]}{ % loop over nodes in previous layer
              \draw[connect,white,line width=1.2] (N\prev-\j) -- (N\lay-\i);
              \draw[connect] (N\prev-\j) -- (N\lay-\i);
              %\draw[connect] (N\prev-\j.0) -- (N\lay-\i.180); % connect to left
            }
          \fi % else: nothing to connect first layer
          
        }
        \path (N\lay-\N) --++ (0,1+\yshift) node[midway,scale=1.5] {$\vdots$};
      }
    \end{scope}

    \begin{scope}[transform canvas={scale=0.7}]
        % Matrix with brackets indicating transformation
    \matrix (m) [matrix of math nodes, nodes={minimum width=0em, minimum height=0em}] at (-1,-14.5)
    {
        \dot{\bar{y}} = \dot{y} \cos{\theta} + \dot{z} \sin{\theta}\\
        \dot{\bar{z}} = -\dot{y} \sin{\theta} + \dot{z} \cos{\theta}\\
        a\\
        \dot{\theta}\\
    };
    % Matrix with brackets indicating transformation
    \matrix (m) [matrix of math nodes, nodes={minimum width=0em, minimum height=0em}] at (15,-15)
    {
        \new{\Delta{\bar{y}}}  \\
        \new{\Delta{\dot{\bar{y}}}} \\
        \new{\Delta{\bar{z}}}  \\
        \new{\Delta{\dot{\bar{z}}}} \\
        \new{\Delta{\cos{\bar{\theta}}}} \\
        \new{\Delta{\sin{\bar{\theta}}}}\\
    };
    \end{scope}

    \begin{scope}[yshift=-9.5cm] 
      \message{^^JNeural network, shifted}
      \readlist\Nnod{2,5,5,4} % array of number of nodes per layer
      \readlist\Nstr{n,m,m,m,k} % array of string number of nodes per layer
      \readlist\Cstr{\strut x,a^{(\prev)},a^{(\prev)},a^{(\prev)},y} % array of coefficient symbol per layer
      \def\yshift{0.5} % shift last node for dots
      
      \message{^^J  Layer}
      \foreachitem \N \in \Nnod{ % loop over layers
        \def\lay{\Ncnt} % alias of index of current layer
        \pgfmathsetmacro\prev{int(\Ncnt-1)} % number of previous layer
        \message{\lay,}
        \foreach \i [evaluate={\c=int(\i==\N); \y=\N/2-\i-\c*\yshift;
                     \index=(\i<\N?int(\i):"\Nstr[\lay]");
                     \x=2*\lay; \n=\nstyle;}] in {1,...,\N}{ % loop over nodes
          % NODES
          \node[node \n] (N\lay-\i) at (\x,\y) {};
          
          % CONNECTIONS
          \ifnum\lay>1 % connect to previous layer
            \foreach \j in {1,...,\Nnod[\prev]}{ % loop over nodes in previous layer
              \draw[connect,white,line width=1.2] (N\prev-\j) -- (N\lay-\i);
              \draw[connect] (N\prev-\j) -- (N\lay-\i);
              %\draw[connect] (N\prev-\j.0) -- (N\lay-\i.180); % connect to left
            }
          \fi % else: nothing to connect first layer
          
        }
        \path (N\lay-\N) --++ (0,1+\yshift) node[midway,scale=1.5] {$\vdots$};
      }
    \end{scope}

\end{tikzpicture}
    \caption{Illustration of %a car's dynamics demonstrating 
    translational and rotational invariance in car dynamics. By applying Cartan's moving frame method, coordinates are transformed to position the car at the origin with a neutral orientation. The function $\rho$ reduces the system's state to a lower-dimensional space, without losing essential dynamics information. %Subsequently, 
    The dynamics are learned in these modified coordinates via a smaller neural network (bottom) compared to the usual NN (middle). The NN's output is then reconverted to original coordinates using $\phi_{\gamma(x)}^{-1}$. \new{As usual, the neural network training is simplified by outputting the difference between the next and current states, represented as $\Delta F(x, u)$ and $\Delta \bar{F}(x, u)$ in Section \ref{sec:learning}. This is an illustration of a single car. In the experiments, additional symmetries are observed due to the presence of two cars; however, the second car has been omitted here for clarity.}
}
    \label{fig:symmetry-explanation}
\end{figure}

%We first describe the symmetries applied to learn translation- and rotation-invariant dynamics for the two car models.
Each car has state $ x = (y, z, v_y, v_z, h_y, h_z)$, corresponding to the $y$ position, $z$ position, $y$ component of velocity, $z$ component of velocity, cosine of heading angle, and sine of heading angle.
\new{
We use the group $G = \mathrm{SE}(2)$ to represent the symmetries. See Example~\ref{example:car} for the action of $G$ on the state space and the solutions for $\gamma$ and $\rho$.
}
The action of $G$ on the control input space is the identity: $\psi_g(u) = u$ for all $g \in G$.
The goal dynamical model satisfies the trivial transformation group $(\mathbb{R}^6, +)$ which acts on the goal by $\phi_\delta(g) = g + \delta$. There is no input to act on.
The $a$ component of $\phi$ has 6 components, leaving the $b$ component with 0.
%Effectively, we remove the goal states entirely.

% We first describe the symmetries applied to learn translation- and rotation-invariant dynamics for the two car models.
% We parameterize the symmetry group $G = SE(2)$ by a translation $(y^\prime, z^\prime)$ and a rotation $\theta^\prime$.
% Define the rotation matrix
% \begin{equation} \label{eq:rotation-mat}
%     R_{\theta^\prime} = \begin{bmatrix}
%         \cos \theta^\prime & -\sin \theta^\prime \\
%         \sin \theta^\prime & \cos \theta^\prime
%     \end{bmatrix}.
% \end{equation}
% %With this, we 
% To define the action of $G$ on the state space,
% %as follows.
% let $(y^\prime, z^\prime, \theta^\prime) \in G$ and $x$ be the 6-dimensional state of the car:
% \begin{align}
%     \phi_{(y^\prime, z^\prime, \theta^\prime)}(x) &= \begin{bmatrix}
%         R_{\theta^\prime} & 0 & 0 \\
%         0 & R_{\theta^\prime} & 0 \\
%         0 & 0 & R_{\theta^\prime}
%     \end{bmatrix}
%     x 
%     + \begin{bmatrix}
%         \begin{bmatrix}
%             y^\prime \\ z^\prime
%         \end{bmatrix} \\ 0 \\ 0
%     \end{bmatrix}
% \end{align}

\begin{figure*}[tp]
    \centering
    \includegraphics[width=0.33 \textwidth]{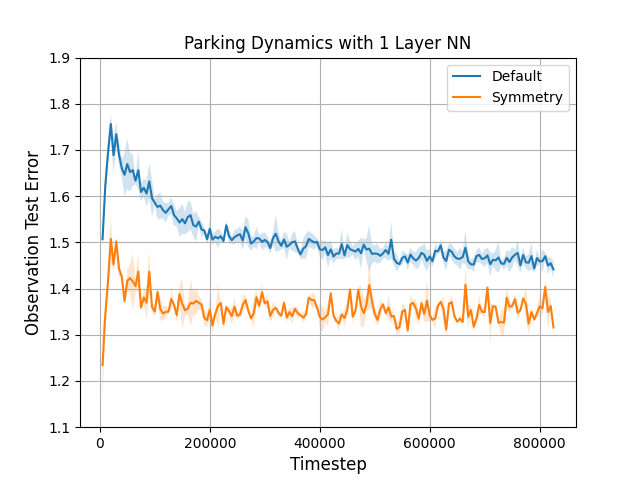}
    \includegraphics[width=0.33 \textwidth]{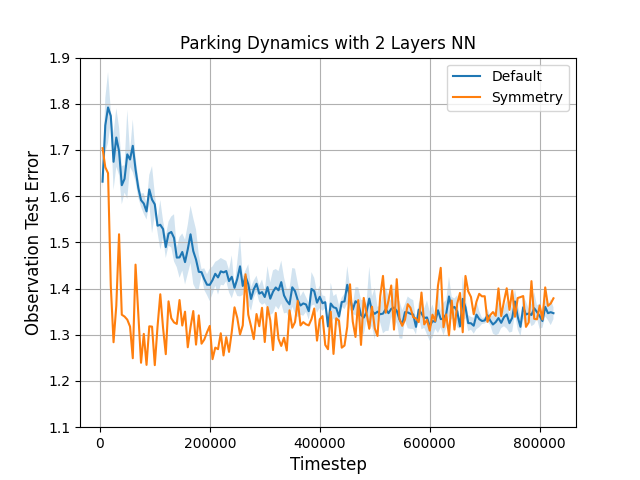}
    \includegraphics[width=0.32 \textwidth]{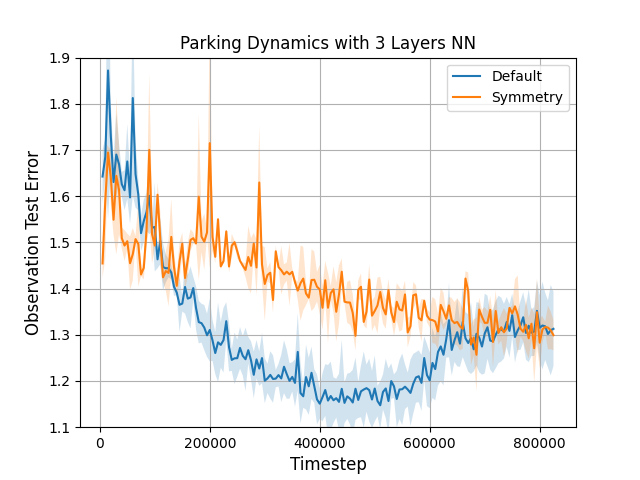}
    \caption{Comparison of learning the dynamics with and without using symmetry in parking scenario with different NN architectures. The y-axis (observation error) is the error on the test dataset. Mean and standard deviation reported over 4 runs.}
    \label{fig:parking-results}
\end{figure*}

% \begin{figure*}[tp]
%     \centering
%     \includegraphics[width=0.33 \textwidth]{img/parking_low_learning_rate/1Layer/Parking Dynamics with 1 Layer NN_observation_error.png}
%     \includegraphics[width=0.33 \textwidth]{img/parking_low_learning_rate/2Layer/Parking Dynamics with 2 Layers NN_observation_error.png}
%     \includegraphics[width=0.32 \textwidth]{img/parking_low_learning_rate/3Layer/Parking Dynamics with 3 Layers NN_observation_error.png}
%     \caption{Comparison of learning the dynamics with and without using symmetry in parking scenario with different NN architectures. The y-axis (observation error) is the error on the test dataset. Mean and standard deviation reported over 4 runs.}
%     \label{fig:parking-results}
% \end{figure*}

\begin{figure*}[tp]
    \centering
    \includegraphics[width=0.33 \textwidth]{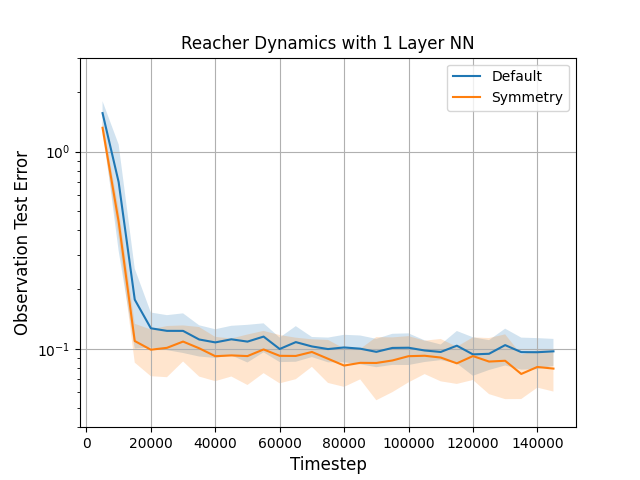}
    \includegraphics[width=0.33 \textwidth]{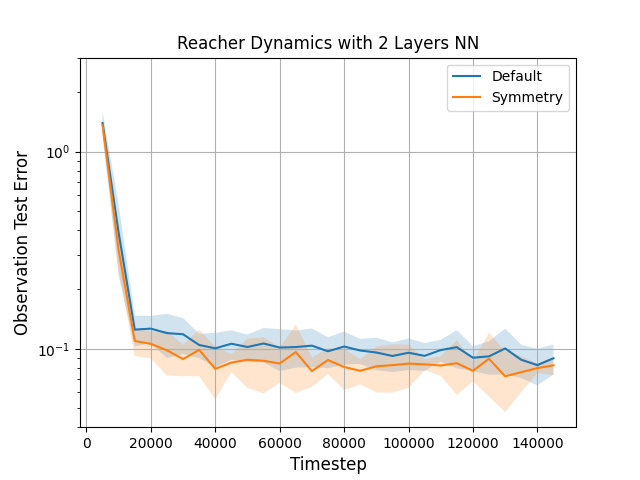}
    \includegraphics[width=0.32 \textwidth]{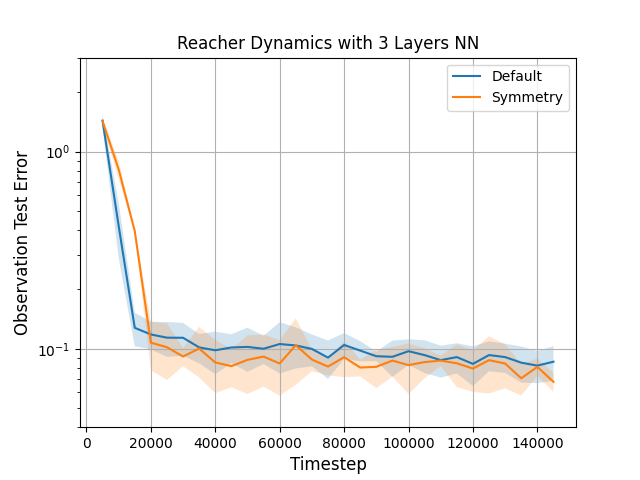}
    \caption{Comparison of learning the dynamics with and without using symmetry in reacher environment with different NN architectures. The y-axis (observation error) is the error on the test dataset. Mean and standard deviation reported over 4 runs.}
    \label{fig:reacher-results}
\end{figure*}

Using the above transformation groups applied to each car and goal, the transformation group applied to the joint state of the environment gives a $\rho$ which \new{transforms the input state from a 24-dimensional vector to a 4-dimensional one}.
Thus, we learn a neural network with significantly reduced input size. Note that the computation of these functions only depends on the symmetry and not on the dynamics.
%, so the 
The
dynamical model need not be known apriori and can be learned using the data;
%. This procedure is explained in 
see Figure~\ref{fig:symmetry-explanation}.

The three neural network configurations %that we 
used in the experiments are: 1 hidden layer with size 128, 2 hidden layers with size 128, and 3 hidden layers with size 128.
%In the method 
When using symmetry, the input size is 8 (reduced coordinate size for each car of 2 and control input size of 2 for each %of the two cars
car), and the output size is 24.
In the method not using symmetry, the input size is 28 (state space size of 24 and control input size of 2 for each of the two cars), and the output size is 24.

The results of observation error vs. number of parameter updates are shown in Figure~\ref{fig:parking-results}.
At the lower number of parameters, the dynamical model training with symmetry outperforms the default method without symmetry by achieving a lower observation error.
\new{When there are two hidden layers, the method with symmetry learns faster than the default method without symmetry, although eventually the method without symmetry catches up in performance.}
% When there are two hidden layers, the method without symmetry learns slower than does the method with symmetry, but ultimately catches up in performance. \new{This suggests that the method with symmetry is faster to learn in this case too}.
When there are three hidden layers, the method without symmetry exceeds the performance of the method with symmetry.
\new{
This suggests that the relative benefit of leveraging symmetry compared to the benefit from the number of parameters diminishes as the number of neural network parameters increases.
One possible cause for this is that our method of leveraging symmetry removes redundant information.
% but the neural network might be able to exploit this redundant information when it has a sufficiently high number of parameters.
}
\newtwo{
This redundant information is a transformed version of the other available information, and might be more immediately useful to the network. As a result, the neural network does not have to learn these transformations on its own, making the redundant information beneficial when a sufficiently high number of parameters is available.
}

% \new{The observed phenomenon may be attributed to the diminishing impact of symmetry as the number of parameters increases, which %diminishes 
% reduces the relative benefit derived from symmetry compared to the benefit from the number of parameters, thereby not yielding a significant difference. When symmetry is not utilized, the network is provided with additional, potentially unnecessary \neelay{(redundant instead of unnecessary?)} information about the states. If the network has sufficient parameters, this information can be effectively exploited, leading to improved outcomes. Consequently, this might be the reason better results are observed in the three-layer case when symmetry is not employed. Another potential explanation is the use of identical hyperparameters for both the default and symmetric cases. If hyperparameters were 
% %to be 
% tested and optimized separately for each case, the distinction between the 
% %two 
% methods might become more evident.} \neelay{(Yeah I agree with Murat that this last sentence invites too many questions.)}
% % catches up with two hidden layers, although it learns slower than the method with symmetry does, and exceeds the performance of the method with symmetry when there are three hidden layers. 
% % {\color{I would start the sentence with "When there are three hidden layer,..." rather than ending it with this crucial clause.}}

\subsection{Reacher}
In this experiment, we consider the standard reacher environment in OpenAI Gym.
The dynamics 
%of the reacher 
exhibit rotational symmetry with the first joint angle, \new{but the dynamics of the target do not since the target is fixed.}
We use symmetry reduction to learn model dynamics which are rotation-invariant.

The state of the reacher is 11-dimensional, with $x_1$ and $x_2$ representing the cosines of the first and second joint angles, $x_3$ and $x_4$ representing the sines of the joint angles, $x_5$ and $x_6$ representing the x- and y-coordinates of the target, $x_7$ and $x_8$ representing the angular velocities of the first and second arm, $x_9$ and $x_{10}$ representing the distance between the reacher fingertip and the target along the x- and y-axes, and $x_{11}$ equaling %the constant 
0, representing the $z$ coordinate of the fingertip.

We now describe the transformation group applied to learn a rotation-invariant dynamical model. %for the reacher.
For $x_5, x_6$, and $x_{11}$, we apply the same symmetry used in Section~\ref{subsec:parking} to describe constants.
Let $G$ be parameterized by $g = (\theta^\prime, \delta_1, \delta_2, \delta_3) \in \mathbb{R}^4$, where $\theta^\prime$ represents the angle by which the reacher is rotated, and $\delta$ represents translations of the target position and $x_{11}$.
We define the action of $G$ on the state space as:
\begin{equation}
    \phi_g(x) =
    \begin{bmatrix}
        \cos(\theta^\prime) x_1 - \sin(\theta^\prime) x_3 \\
        x_2 \\
        \sin(\theta^\prime) x_1 + \cos(\theta^\prime) x_3 \\
        x_4 \\
        \cos(\theta^\prime)(x_5 + \delta_1) - \sin(\theta^\prime)(x_6 + \delta_2) \\
        \sin(\theta^\prime)(x_5 + \delta_1) + \cos(\theta^\prime)(x_6 + \delta_2) \\
        x_7 \\
        x_8 \\
        \cos(\theta^\prime)(x_9 - \delta_1) - \sin(\theta^\prime)(x_{10} - \delta_2) \\
        \sin(\theta^\prime)(x_9 - \delta_1) + \cos(\theta^\prime)(x_{10} - \delta_2) \\
        x_{11} + \delta_3
    \end{bmatrix}.
\end{equation}

The action of $G$ on the control input is the identity: $\psi_g(u) = u$ for all $g \in G$.
%To define the symmetry reduction, we 
We define $\mathcal{C}$ by the states $x$ for which $x_1 = 1$, $x_3 = 0$, $x_5 = 0$, $x_6 = 0$, and $x_{11} = 0$.
The solution for the moving frame $\gamma(x)$ is 
\begin{equation*}
    \gamma(x) = \begin{bmatrix}
        \mathrm{arctan2}(-x_3, x_1) \\
        -x_5 \\
        -x_6 \\
        -x_{11}
    \end{bmatrix}
\end{equation*}
and the solution for the group inverse of $\gamma(x)$ is 
\begin{equation*}
    \gamma(x)^{-1} = \begin{bmatrix}
        \mathrm{arctan2}(x_3, x_1) \\
        x_1 x_5 + x_3 x_6 \\
        -x_3 x_5 + x_1 x_6 \\
        x_{11}
    \end{bmatrix}.
\end{equation*}
We now plug $\gamma$ into \eqref{eq:rho} to solve for $\rho$:
\begin{equation}
    \rho(x) = \begin{bmatrix}
        x_2 \\ x_4 \\ x_7 \\ x_8 \\ 
        x_1 (x_9 + x_5) + x_3 (x_{10} + x_6) \\
        -x_3 (x_9 + x_5) + x_1 (x_{10} + x_6)
    \end{bmatrix}.
\end{equation}

Using this symmetry group, the input size of the neural network is 8 (reduced coordinate size 6 and control input size of 2), and the output size is 11.
%In the method 
When not using symmetry, the input size is 13 (state space size of 11 and control input size of 2), and the output size is 11.
The three neural network configurations that we used in our experiments are: 1 hidden layer with size 64, 2 hidden layers with size 64, and 3 hidden layers with size 64. We used 64 neuron size compared to 128 in ``Parking" because of the simplicity of this problem.

The results of observation error vs. number of parameter updates are shown in Figure~\ref{fig:reacher-results}.
In all cases, the dynamical model training with symmetry achieves slightly better performance than the default method without symmetry by achieving a lower observation error at the end of the training.
This is most evident when the neural network has the least number of parameters, in the 1-layer case. 
\section{Conclusion}

This study highlights the benefits of %incorporating 
exploiting
symmetry 
%into 
when
learning dynamical models.
% the dynamical models used in model-based reinforcement learning. 
By focusing on the symmetry inherent in dynamics, independent of the reward function, we broaden the range of problems where symmetry can be effectively utilized. We investigated symmetries such as rotation and translation symmetries through Cartan's method of moving frames and our proposed method proves to be effective in leveraging these symmetries to learn more accurate models, particularly at low numbers of model parameters. 
The experiments %conducted 
confirm the potential of our approach to enhance learning efficiency.
%in MBRL settings. We compare our method with varying parameter sizes and see a higher benefit of utilizing symmetry in smaller networks. 
\new{Future work involves training policies using the proposed symmetry-enhanced method and evaluating their performance. Additionally, a theoretical framework that includes discrete symmetries can be developed.
}
%, further expanding the scope of our approach.}

%%%%%%%%%%%%%%%%%%%%%%%%%%%%%%%%%%%%%%%%%%%%%%%%%%%%%%%%%%%%%%%%%%%%%%%%%%%%%%%%
\section{Acknowledgements}
We thank Sergey Levine for pointing to relevant references and for discussions on 
%the 
preliminary results,
%Additionally, we thank 
and Hussein Sibai for discussions on symmetry reduction prior to this project. %that eventually led to this formulation of the paper.

%%%%%%%%%%%%%%%%%%%%%%%%%%%%%%%%%%%%%%%%%%%%%%%%%%%%%%%%%%%%%%%%%%%%%%%%%%%%%%%%
\renewcommand*{\bibfont}{\footnotesize}
\printbibliography
\end{document}